\documentclass{article}
\usepackage{ijcai16}

\usepackage{times}

\usepackage{subcaption}
\usepackage{url}
\usepackage{amssymb}
\usepackage{amsmath}
\usepackage{amsthm}
\usepackage{mathrsfs}
\usepackage{bm}
\usepackage{mathtools}
\usepackage{algorithm}
\usepackage[noend]{algorithmic}

\newtheorem{definition}{Definition}
\newtheorem{theorem}{Theorem}
\newtheorem{lemma}{Lemma}
\newtheorem{remark}{Remark}
\newtheorem{corollary}{Corollary}

\newcommand\numberthis{\addtocounter{equation}{1}\tag{\theequation}}
\newcommand{\norm}[1]{\left\lVert#1\right\rVert}
\newcommand{\alg}{}

\title{Theoretically-Grounded Policy Advice from Multiple Teachers in Reinforcement Learning Settings with Applications to Negative Transfer}
\author{Yusen Zhan$^{1}$, Haitham Bou Ammar$^{2}$, and Matthew E. Taylor$^{1}$ \\
	$^{1}$Washington State University, Pullman, Washington  \\
	$^{2}$Princeton University, Princeton, New Jersey\\
	yusen.zhan@wsu.edu, hammar@princeton.edu, taylorm@eecs.wsu.edu
}

\begin{document}

\maketitle

\begin{abstract}
Policy advice is a transfer learning method where a student agent is able to learn faster via advice from a teacher. However, both this and other reinforcement learning transfer methods have little theoretical analysis. This paper formally defines a setting where multiple teacher agents can provide advice to a student and introduces an algorithm to leverage both autonomous exploration and teacher's advice. Our regret bounds justify the intuition that good teachers help while bad teachers hurt. Using our formalization, we are also able to quantify, for the first time, when negative transfer can occur within such a reinforcement learning setting.
\end{abstract}

\section{Introduction}
Reinforcement Learning (RL) has become a popular framework for autonomous behavior generation from limited feedback~\cite{sutton1998introduction}. Typical RL methods learn in isolation increasing their learning times and sample complexities. Transfer learning aims to significantly improve learning by providing informative knowledge from an external source. The source of such knowledge varies from source agents to humans providing advice~\cite{erez2008does,taylor2009transfer}. In this paper, we focus on a framework referred to as action advice or the \emph{advice model}~\cite{torrey2013teaching}. Here, the agent (i.e., student), learning in a task, has access to a teacher (another agent or human) which can provide action suggestions to facilitate learning.
Given ``good-enough'' teachers, such advice models have shown multiple benefits over standard RL techniques. For example, others~\cite{torrey2013teaching,LIP610219} show reduced learning times and sample complexities for successful behavior. 

These methods, however, suffer from two main drawbacks. First, validation results are empirical in nature and not formally-grounded. We do not have fundamental understanding of these methods. Consequently, it is difficult to formally comprehend why these methods work. 
Second, most of these techniques require the availability of a ``good-enough'' (optimal) teacher to benefit the student. Unfortunately, access to such teachers is difficult in a variety of complex domains, reducing the applicability of policy advice in real-world settings. 

In this paper, we remedy the aforementioned drawbacks by proposing a new framework for policy advice. Our method formally generalizes current single-teacher advice models to the multi-teacher setting. Our algorithm also remedies the need for optimal teachers by exploiting both the student's and the teacher's knowledge. Even if the teacher is not optimal, a student, using our algorithm, is still capable of acquiring optimal behavior in a task; a property not supported by some state-of-the-art methods, e.g., learning from demonstration. We theoretically and empirically analyze the performance of the proposed method and derive, for the first time, regret bounds quantifying the successfulness of action advice. We also provide theoretical justification for current methods (i.e., single-teacher models) as special case of our formulation in the appendix. Our contributions can be summarized as: 
\begin{itemize}
	\item defining (formally) multi-teacher advice models, 
	\item introducing novel algorithms leveraging teacher and student knowledge, 
	\item deriving the regret analysis showing reduced sample complexities,
	\item deriving theoretical guarantees for single teacher advice models, and
	\item quantifying negative transfer under such advice model. 
\end{itemize}

Interestingly, these theoretical results justify a well-known intuition inherent to advice models: ``good teachers help while bad teachers hurt.'' The results  show that students can still achieve optimal behavior when being advised by bad teachers. They, however, pay an extra cost in terms of their learning times or sample complexities, relative to an optimal teacher. This should inspire researchers to adopt high quality teacher policies or avoid ``bad teachers'' if possible.

Given our formalization, we also derive a relation to negative transfer. We quantify, for the first time, the occurrence of negative transfer in action advice models, shedding the light on failure modes of these methods. Consequently, these results yield two claims about transfer learning. First, high quality transfer knowledge may still cause negative transfer when the target algorithm is able to outperform the source knowledge. Second, expert knowledge is important for the researchers to determine whether or not to transfer because evaluation of the transfer knowledge is usually expensive (it is equivalent to evaluating the teacher policy in the target MDP).


\section{Preliminaries}
\subsection{Online Reinforcement Learning \& Regret Model}
In RL, an agent must sequentially select actions to maximize its total expected return. Such problems are formalized as a Markov decision Process (MDP), defined as $\mathcal{M}= \langle \mathcal{S},\mathcal{A}, \mathcal{P}, \mathcal{R} \rangle$, where $\mathcal{S}$ and $\mathcal{A}$ denote the finite state and actions spaces with a total size of $|\mathcal{S}|$ and $|\mathcal{A}|$ respectively, $\mathcal{P}: \mathcal{S} \times \mathcal{A} \times \mathcal{S} \rightarrow [0,1]$ represents the probability transition kernel describing the task dynamics, and $\mathcal{R}: \mathcal{S} \times \mathcal{A} \rightarrow \mathbb{R}$ is the reward function quantifying the performance of the agent. The total expected return of an agent following an algorithm, $\mathcal{G}$, to compute the optimal action-selection rule from a starting state $\bm{s} \in \mathcal{S}$ after $T$ time steps is defined as: 
\begin{equation}
\label{Eq:return}
\mathcal{R}^{\mathcal{G}} (\bm{s},T) = \mathbb{E}\left[\sum_{t=0}^{T} \mathcal{R} (\bm{s}_{t},\bm{a}_{t})\right],
\end{equation} 
with $\bm{s}_{t} \in \mathcal{S}$ and $\bm{a}_{t} \in \mathcal{A}$. The goal is to determine an optimal policy, $\pi^{\star}: \mathcal{S} \rightarrow \mathcal{A}$ that maximizes the total expected return.

\textbf{Regret Model:} Similar to standard online learning, we quantify the performance of the algorithm, $\mathcal{G}$, by measuring its regret with respect to optimality. We define the regret of a state $\bm{s}$ after $T$ time steps in terms of the expected reward as: 
\begin{equation}
\label{Eq:regret}
\bm{\Delta}^{\mathcal{G}}(\bm{s},T)= \bm{\lambda}^{\star}T - \mathcal{R}^{\mathcal{G}} (\bm{s},T),
\end{equation}
where $\bm{\lambda}^{\star}$ is the optimal reward acquired by following an optimal algorithm $\mathcal{G}^{\star}$ at each time step. In the general case when no reachability assumptions are imposed, it is easy to construct MDPs in which algorithms suffer high regret. Following \citeauthor{puterman2005markov} \shortcite{puterman2005markov}, we remedy this problem by considering weakly-communicating MDPs\footnote{Please note that weakly-communicating MDPs are considered the most general among subclasses of MDPs, see~\citeauthor{puterman2005markov} \shortcite{puterman2005markov}.} defined as follows.
\begin{definition}
	\label{def:weakcommun}
	An MDP is called \textit{weakly communicating} in such a case where the state set $\mathcal{S}$ can be decomposed into two subsets, $\mathcal{S}_{1}$ and $\mathcal{S}_{2}$. In $\mathcal{S}_{1}$ any state is reachable from every other state under a deterministic policy, $\pi$, while states in $\mathcal{S}_{2}$ are transient under all policies.
\end{definition}

The optimal gain, $\bm{\lambda}^{\star}$ in Equation~\ref{Eq:regret}, is state independent. That is, any $\bm{s} \in \mathcal{S}$, shares the same optimal expected reward~\cite{puterman2005markov}, which can be solved for using: 
\begin{equation*}
\bm{h}^{\star}+\bm{\lambda}^{\star}\bm{e}=\max_{\bm{a} \in \mathcal{A}}\{\mathcal{R}(\bm{s},\bm{a})+\mathcal{P}^\mathsf{T}_{\bm{s},\bm{a}} \bm{h}^\star\},
\end{equation*}
where $\bm{h}^{\star}$ is an $|\mathcal{S}|$ dimensional vector typically referred to as the bias vector, $\mathcal{P}_{\bm{s},\bm{a}}$ denotes the probability to transition from $\bm{s}$ applying action $\bm{a}$, and  $\bm{e} \in \mathbb{R}^{|\mathcal{S}|}$ is a unit dimensional vector. When needed, we explicitly write the dependency of $\bm{\lambda}^{\star}$ and $\bm{h}^{\star}$ as $\bm{h}^{\star}(\bm{s}; \mathcal{M})$ and $\bm{\lambda}^{\star}(\mathcal{M})$. We also define the span of $\bm{h}$ as: $
\text{sp}(\bm{h})= \max_{\bm{s} \in \mathcal{S}} \bm{h}(\bm{s})- \min_{\bm{s} \in \mathcal{S}} \bm{h}(\bm{s}).$

Finally, we follow \citeauthor{bartlett2009regal} \shortcite{bartlett2009regal} to define reachability in weakly communicating MDPs using the \emph{one-way diameter}: 
$\text{diam}_{\text{one-way}}(\mathcal{M}) = \max_{\bm{s} \in \mathcal{S}} \min_{\pi} T^{\pi}_{\bm{s}_1\to \bar{\bm{s}}}$,
with $T^{\pi}_{\bm{s}_1\to \bar{\bm{s}}}$ being the expected number of steps needed for reaching $\bar{\bm{s}}=\arg\max_{\bm{s} \in \mathcal{S}} \bm{h}^{\star}(\bm{s}; \mathcal{M})$ from $\bm{s}_{1} \in \mathcal{S}$. 

\subsection{Algorithms for Weakly-Communicating MDPs}
REGAL.C is an on-line algorithm for weakly communicating MDPs developed by \citeauthor{bartlett2009regal}~\shortcite{bartlett2009regal}. The basic idea is that the REGAL.C can estimate the true MDP with high probability in order to learn an $\epsilon$-optimal policy with high probability. Let $N(\bm{s},\bm{a},\bm{s}^{\prime};t)$ be the number of state-action-state triples $(\bm{s},\bm{a},\bm{s}^{\prime})$ that have been visited at time $t$. Further, let $t_i$ to denote the initial time of the iteration $i$. For brevity, we use  $N_i(\bm{s},\bm{a},\bm{s}^{\prime})$ and  $N_i(\bm{s},\bm{a})$ to denote $N(\bm{s},\bm{a},\bm{s}^{\prime};t_i)$ and $N(\bm{s},\bm{a};t_i)$ at iteration $i$. We also use $v_i(\bm{s},\bm{a})=N_{i+1}(\bm{s},\bm{a})-N_{i}(\bm{s},\bm{a})$ to denote the number of times a state-action pair $(\bm{s},\bm{a})$ is visited during iteration $i$. For each iteration $i$, REGAL acquires a dataset $\mathcal{D}_i$ as input and updates the transition probability (see Equation~\ref{eq:updatepro}). It then constructs a set of MDPs $\mathcal{M}^i$ to select from using $\max \lambda^*(M)\text{   s.t.   }sp(h^*(M)) \leq H$, where $H$ is the upper bound on the span $\text{sp}(\bm{h}^*(\mathcal{M}))$. Given the MDP, REGAL.C uses value iteration for acquiring the optimal policy. These steps are summarized in Algorithm~\ref{alg:REGAL.C}.
\begin{algorithm}[t]
	\caption{REGAL.C: Constrained Optimization}
	\textbf{Input}: parameter $H$, dataset $\mathcal{D}_i$ and current time $T$ \\
	\textbf{Output}: $\hat{\pi}_{i+1}$
	\label{alg:REGAL.C}
	\begin{algorithmic}[1]
		\STATE $t_i=$ current time $T$
		\STATE Use $\mathcal{D}_i$ to update the state transition probabilities by \begin{equation}
		\label{eq:updatepro}
		\hat{P}^{t}_{\bm{s},\bm{a}}(\bm{s}^{\prime})=\frac{N(\bm{s},\bm{a},\bm{s}^{\prime};t)}{\max\{N(\bm{s},\bm{a};t),1\}}
		\end{equation}
		\STATE With $t=t_i$, $\mathcal{M}^i$ is the set of MDPs s.t.
		\[
		\norm{P_{\bm{s},\bm{a}}-\hat{P}^t_{\bm{s},\bm{a}}}_1\leq \sqrt{\frac{12|\mathcal{S}|\log(2|\mathcal{A}|t/\delta)}{\max\left\{N(\bm{s},\bm{a},\bm{s};t),1\right\}}}
		\]
		\STATE Select $M^i \in \mathcal{M}^i$ by following optimization equation over $\forall M \in \mathcal{M}^i$,
		\begin{align*}
		\max \bm{\lambda}^{\star}(M)\text{   s.t.   }\text{sp}(\bm{h}^*(\mathcal{M})) \leq H
		\end{align*}
		\STATE $\hat{\pi}_{i+1}$=average reward optimal policy for $M^i$ (value iteration)
		\RETURN $\hat{\pi}_{i+1}$
	\end{algorithmic}
\end{algorithm}

\subsection{Single Teacher Advice Model}
The single teacher advice model is a framework in which a student learning in an environment benefits from a teacher's advice to speed-up learning. We define such a framework as the tuple of $\langle \pi^{\mathfrak{T}}, \bm{b}, \mathfrak{S}, \bm{f}_{d} \rangle$. Here, $ \pi^{\mathfrak{T}}$ denotes the teacher's policy, $\bm{b}$ represents the budget constraining the teacher's advice, $\mathfrak{S}$ is the student, and $\bm{f}_{d}$ is a function controlling the advice from the teacher to the student. Apart from considering single teacher models, previous work assumed optimal teachers where students always execute recommended actions. It is easy to construct complex settings in which access to optimal teachers is difficult. Consequently, we extend these works to the more realistic settings of sub-optimal teachers, as we detail later. 

\section{Multiple Teacher Advice Model}
In this section we start by extending the single teacher model of~\citeauthor{torrey2013teaching}~\shortcite{torrey2013teaching} to the multiple \emph{non-optimal} teacher setting. Our advice model for $m$ teachers is defined as the tuple $\langle \Pi, \mathcal{B}, \mathfrak{S}, \bm{f}_{d}\rangle$, where $\Pi=\{\pi^{\mathfrak{T}_{1}}_1,\pi^{\mathfrak{T}_{2}}_2,\dots,\pi^{\mathfrak{T}_{m}}_m\}$ is the set of $m \in \mathbb{N}$ teacher policies, and $\mathcal{B}=\{\bm{b}_1,\bm{b}_2,\dots,\bm{b}_m\}$ denotes the set of budgets. It is easy to see that in case $\Pi=\{\pi^{\mathfrak{T}}\}$ and $\mathcal{B}=\{\bm{b}\}$, we can easily recover the special case single teacher model. We also generalize the work of~\citeauthor{torrey2013teaching}~\shortcite{torrey2013teaching} by making no restrictive assumptions on the optimality of any of the teachers. We measure the performance of the teacher with respect to a base policy $\pi^{\mathfrak{B}}$ in terms of regret: 
\begin{definition}
	Given a teacher's policy, $\pi^{\mathfrak{T}} \in \Pi$, and a base policy $\pi^{\mathfrak{B}}$, then the regret of following $\pi^{\mathfrak{T}}$ is related to that acquired by following $\pi^{\mathcal{B}}$ using: 
	\begin{equation*}
	\bm{\Delta}^{\mathfrak{T}}(\bm{s},T)=\rho\bm{\Delta}^{\mathfrak{B}}(\bm{s},T),
	\end{equation*}
	where $\rho \geq 0$ denotes the regret ratio, $\bm{\Delta}^{\mathfrak{T}}(\bm{s},T)=\bm{\lambda}^{\star}T - \mathcal{R}^{\mathfrak{T}}(\bm{s},T)$ and $\bm{\Delta}^{\mathfrak{B}}(\bm{s},T)=\bm{\lambda}^{\star}T - \mathcal{R}^{\mathfrak{B}}(\bm{s},T).$
\end{definition}
The above definition captures the three interesting cases quantifying the performance of an advice-based algorithm. If the teacher is optimal, i.e., when $\bm{\Delta}^{\mathfrak{T}} (\bm{s},T)=0$, $\rho$ is also $0$. In case  $0 < \rho \leq 1$, then $\bm{\Delta}^{\mathfrak{T}}(\bm{s},T) \leq \bm{\Delta}^{\mathfrak{B}}(\bm{s},T)$ indicating the the teacher's policy is at least as good as the base policy $\pi^{\mathfrak{B}}$. Finally, when $\rho >1$, $\bm{\Delta}^{\mathfrak{T}}(\bm{s},T) > \bm{\Delta}^{\mathfrak{B}}(\bm{s},T)$ implying the underperformance of the teacher. Consequently, with the correct choice of the teacher by $\rho$ one can still achieve successful advice even in such a generalized setting.

\section{Efficient Multi-teacher Advice}\label{Sec:New}
In this section, we propose a new algorithm which combines the advice policy and the MDPs information collected so far. This allows for an accurate framework outperforming state-of-the-art techniques for policy advice. On a high level, our algorithm consists of three main steps. First, a combined policy is constructed based on multiple teachers. Second, data depending on both teacher's advice as well as MDP information is collected. Third, a new policy is computed online. 

Next, we outline each of the three steps and describe our novel algorithm. Having achieved an accurate advice model, we then rigorously analyze the theoretical aspects of our method and show a decrease in sample complexities compared to current techniques. 
\subsection{The Grand-Teacher}
Our method of policy advice constructs a grand teacher combining all teacher policies in a meta-policy. To construct the grand-teacher, we use an ensemble method and design two meta-policy variations: online and offline-constructions. Next, we detail each of the two variations. 

\textbf{Online Grand-Teacher:} 
In the \textit{online} construction, whenever the student observes an unvisited state, $\bm{s} \in \mathcal{S}$, each teacher provides its policy advice of the form $\pi_{i}^{\mathfrak{T}_{i}}$, for all $i \in \{1,\dots, m\}$ with $m$ being the total number of teachers. The student then selects and stores the majority action from all teachers for that state $\bm{s}$. As far as budget is concerned, it is easy to see that we only require to know advice for each state in $\mathcal{S}$, thus $\bm{b}_{1}=\dots=\bm{b}_{m}=|\mathcal{S}|$. Though easy to implement and test, the online construction suffers from the potentially unrealistic need for the continuous availability of online teachers. 

\textbf{Offline Grand-Teacher:} To eliminates the need for an online teacher at each visit of a new state, the offline procedure traverses the states in the MDP for constructing the meta-advice policy. The main steps of this construction is summarized in Algorithm~\ref{Algo:ConstructTeacher}.  

\begin{algorithm}[t]
	\caption{Offline Construction of the Meta-Teacher}
	\textbf{Input}: The set of states in the MDP, $\mathcal{S}$. \\
	\label{Algo:ConstructTeacher}
	\begin{algorithmic}[1]
		\WHILE{$\exists s\in \mathcal{S}$ is not visited}
		\STATE Follow a policy in the MDP
		\IF{Current state $\bm{s}$ is not visited}
		\STATE Query all teachers for advice and select action $\bm{a}$ using Majority Vote. 
		\ENDIF
		\ENDWHILE
		\RETURN $\pi^{\text{grand-teacher}}$ 
	\end{algorithmic}
\end{algorithm}
Note that Algorithm~\ref{Algo:ConstructTeacher} is capable of constructing an offline meta-teacher but requires extra exploration in the MDP. We next show that $\mathcal{O}\left(|\mathcal{S}|\log\frac{|\mathcal{S}|}{\bm{\delta}}\right)$ steps are enough to explore each state in the MDP with high probability: 
\begin{theorem}[Sample Complexity]
	If Algorithm~\ref{Algo:ConstructTeacher} independently and uniformly explores each state $\bm{s} \in \mathcal{S}$, then with probability of at least $1-\bm{\delta}$, $\mathcal{O}\left(|\mathcal{S}|\log\frac{|\mathcal{S}|}{\bm{\delta}}\right)$ steps are sufficient to visit each state at least one time. 
\end{theorem}

\subsection{Multi-Teacher Advice Algorithm}
To improve current methods and arrive at a more realistic advice framework, we now introduce our algorithm combining the grand-teacher's policy and information attained by the student from the MDP. 

Our algorithm is based on the following intuition. At the beginning of the learning process, a student requires guidance as it typically has little to no information of the task to be solved. As time progress and the student explores, the MDP can be effectively exploited for successful learning. Unfortunately, such a process is not well modeled using current methods. Here, we remedy this problem by introducing an algorithm which follows the teacher's advice at the very beginning and then switches to a policy computed by an algorithm operating within the MDP. That is, the teacher guides the student at the beginning of the learning process and as the student gathers more experience, the teacher's influence diminishes over time by switching into a policy computed by REGAL.C. The overall procedure is summarized in Algorithm~\ref{alg:learning_dagger}. Note that our algorithm is inspired by DAGGER~\cite{ross2010reduction} in the sense that policies are updated by collecting data using a mixture of action selection rules (i.e., student and teacher policies). Contrary to DAGGER, however, our method collects all trajectories opposed to only collecting inconsistent actions, allowing for more accurate and efficient updates. 

\begin{algorithm}[t]
	\caption{\alg}
	\textbf{Input}: $\pi^{\mathfrak{T}}$ = the grand-teacher policy, $\hat{\pi}_1$=any policy \\
	\textbf{Output}: $\pi$, the $\epsilon$-optimal policy 
	\label{alg:learning_dagger}
	\begin{algorithmic}[1]
		\STATE $T=0$	
		\FOR{$i=1$ to $m$}
		\STATE Let $\pi_{i+1}=\beta_i\pi^T+(1-\beta_i)\hat{\pi}_i$
		\STATE Follow $\pi_i$ until $T_i$-steps
		\STATE Get dataset $D_i$
		\STATE $T=T+T_i$
		\STATE $\hat{\pi}_{i+1}=REGAL.C(D_i,T)$ \COMMENT{See Algorithm \ref{alg:REGAL.C}}
		\ENDFOR
		\RETURN $\pi_{m+1}$
	\end{algorithmic}
\end{algorithm}

To leverage both the teacher's and learned policies, we set a mixed policy of the form $\pi_{i+1}=\beta_i\pi^{\mathfrak{T}}+(1-\beta_i)\hat{\pi}_i$,
for $0\leq \beta_i\leq 1$ to guide the student's dataset collection while allowing the teacher to fractionally control exploration needed to collect data at the next iteration. $\beta$ should typically be set so as to decay exponentially over time. This decreases the student's reliance on the teacher and allows it to exploit the knowledge gathered from the MDP to learn better behaving policies than that of the teacher. It is for this reason that our algorithm, contrary to other methods, does not impose any optimality restrictions on the teacher. Having collected the dataset, Algorithm~\ref{alg:learning_dagger} uses REGAL.C (Algorithm \ref{alg:REGAL.C}) to update $\hat{\pi}_{i}$.

\subsection{Theoretical Guarantees}
In this section we formally derive the regret exhibited by our algorithm. At a high level, we provide two theoretical results. In the first, we consider the general teacher case, while in the second we derive a corollary of the regret for optimal teachers. We show, for the first time, \emph{better than constant} improvements compared to standard learners.

\begin{theorem}
	\label{thm:regalregret}
	Assume Algorithm \ref{alg:learning_dagger} is running for total $T$ steps in a weakly communicating MPD $\mathcal{M}$ starting from an initial state $\bm{s} \in \mathcal{S}$. Let $H$ be a parameter such that $H \geq \text{sp}(\bm{h}^{\star}(\mathcal{M}))$. Then, with a probability of at least $1-\delta$, the total regret is given by: 
	$
	\bm{\Delta}(\bm{s},T)=\mathcal{O}\left((1-\bm{\beta}+\bm{\rho\beta})H|\mathcal{S}|\sqrt{|\mathcal{A}|T\log \frac{|\mathcal{A}|T}{\delta}}\right),
	$
	where $\bm{\beta}\in[0,1]$ such that $1-\bm{\beta}=\max_{1\leq i\leq m}\{1-\beta_i\}$, and $\bm{\rho}\geq 0$ is the ratio between the teacher's regret $\bm{\Delta}^{\mathfrak{T}}$ and the regret exhibited by REGAL.C $\bm{\Delta}^{\text{REGAL.C}}$ such that $\Delta^{\mathfrak{T}}\leq \bm{\rho}\bm{\Delta}^{\text{REGAL.C}}$.
\end{theorem}
\begin{proof}
	Due to the space limits, we provide a proof sketch. The proof is based on the regret bound of REGAL.C. We introduce the regret ratio to reduce the grand-teacher's regret to the REGAL.C's regret. Then, we apply the Hoeffding's inequality to arrive at the statement of the theorem. 
\end{proof}
Theorem \ref{thm:regalregret} implies that the teacher improves learning as long as it is ``good.'' Namely, if $0<\bm{\rho}\leq 1$, $1-\bm{\beta}+\bm{\rho\beta}\leq 1$, $\bm{\beta}\in[0,1]$ which implies the student can enjoy a fraction of REGAL.C's regret. However, if $\bm{\rho}> 1$, $1-\bm{\beta}+\bm{\rho\beta}> 1$, the student suffers more regret than the original REGAL.C algorithm. This justifies our intuition that good teachers assist learning while poor ones hamper learning. Moreover, if there exists prior knowledge that a teacher has poor performance, it would be better off for the student to neglect its advice as it will suffer extra regret.

If the teacher's $\bm{\rho}=0$, we have the following Corollary:
\begin{corollary}
	\label{co:regalregretoptinalteacher}
	If the teacher is optimal, then with at least a probability of $1-\delta$ the total regret is given by:  
	$
	\bm{\Delta}(\bm{s},T)=\mathcal{O}\left((1-\bm{\beta})H|\mathcal{S}|\sqrt{|\mathcal{A}|T \log \frac{|\mathcal{A}|T}{\delta}}\right).
	$
\end{corollary}

\begin{remark}
	Please note that the above theoretical results are more than a constant improvement to the regret. Notice that $\bm{\beta}$ depends on the number of iterations which can be bounded by $|\mathcal{S}|$ and $|\mathcal{A}|$ of the input MDP $\mathcal{M}$ \cite{auer2009near}. Further, $\bm{\rho}$ depends on the input teacher's policy which is also an input to Algorithm \ref{alg:learning_dagger}. Consequently, it can be shown that these regret improvements exceed simple constant bounds. 
\end{remark}

\section{Negative Transfer}
To formalize the relation to negative transfer, we recognize that the regret ratio can be written as:
\begin{equation}
\rho=\frac{\bm{\Delta}^{\mathfrak{T}}(\bm{s},T)}{\bm{\Delta}^{\mathfrak{B}}(\bm{s},T)}=\frac{\bm{\lambda}^{\star}T - \mathcal{R}^{\mathfrak{T}}(\bm{s},T)}{\bm{\lambda}^{\star}T - \mathcal{R}^{\mathfrak{B}}(\bm{s},T)}
\end{equation}
This suggests that we can estimate the ratio by calculating $\bm{\lambda^{\star}}$ and $\mathcal{R}^{\pi}(\bm{s},T)$, given a policy $\pi$. So, we use $$\rho(\pi_1,\pi_2,T)=\frac{\bm{\lambda}^{\star}T - \mathcal{R}^{\pi_1}(\bm{s},T)}{\bm{\lambda}^{\star}T - \mathcal{R}^{\pi_2}(\bm{s},T)}$$ to denote the regret ratio between policy $\pi_1$ and $\pi_2$ until step $T$. At this stage, we define: 
\begin{itemize}
	\item Negative transfer from policy $\pi_1$ to $\pi_2$ until $T$ steps: $\rho(\pi_1,\pi_2,T)>1$.
	\item Positive transfer from policy $\pi_1$ to $\pi_2$ until $T$ steps: $\rho(\pi_1,\pi_2,T)\leq  1$
\end{itemize}

To formalize negative transfer, our goal at this stage is to relate $\rho(\cdot)$ to a metric between source and target tasks. For that sake, we define: 
$
d_t^s(\pi_s)=\hat{\mathcal{R}}^{\pi_s}_s(\bm{s},T)- \hat{\mathcal{R}}^{\pi_s}_t(\bm{s},T),
$
with $\hat{\mathcal{R}}^{\pi_s}_s(\bm{s},T)$ and $\hat{\mathcal{R}}^{\pi_s}_t(\bm{s},T)$ being the agent's estimates of the rewards in the source and the target after $T$ steps. Consequently, an estimate $\hat{\rho}$ to $\rho$ can be derived as: 
\begin{align*}
&\hat{\rho}(\pi_s,\pi_t,T)\\
&=\frac{\bm{\lambda}^{\star}T - \hat{\mathcal{R}}^{\pi_s}_t(\bm{s},T)}{\bm{\lambda}^{\star}T - \hat{\mathcal{R}}^{\pi_t}_t(\bm{s},T)} \\
&=\frac{\bm{\lambda}^{\star}T +\left(\hat{\mathcal{R}}^{\pi_s}_s(\bm{s},T)- \hat{\mathcal{R}}^{\pi_s}_t(\bm{s},T)\right)- \hat{\mathcal{R}}^{\pi_s}_s(\bm{s},T)}{\bm{\lambda}^{\star}T -\hat{\mathcal{R}}^{\pi_t}_t(\bm{s},T)} \\
&=\frac{\bm{\lambda}^{\star}T +d_t^s(\pi_s)- \hat{\mathcal{R}}^{\pi_s}_s(\bm{s},T)}{\bm{\lambda}^{\star}T -\hat{\mathcal{R}}^{\pi_t}_t(\bm{s},T)}.
\end{align*}

$\hat{\mathcal{R}}^{\pi_s}_s(\bm{s},T)$ and $\hat{\mathcal{R}}^{\pi_t}_t(\bm{s},T)$ can be bounded by the Empirical Bernstein bound \cite{audibert2007tuning}. With a probability $1-\delta$, we have  
\begin{equation*}
\left| \hat{\mathcal{R}}^{\pi_s}_s(\bm{s},T) -\mathbb{E}_{\pi_s}\left[\sum_{t=0}^{T} \mathcal{R}_s (\bm{s}_{t},\bm{a}_{t})\right] \right| \leq \epsilon_1,
\end{equation*}
with $\epsilon_1=\bar{\sigma}\sqrt{\frac{2\log(3/\delta)}{n_s}}+\frac{6R_{max}\log(3/\delta)}{n_s}$, $\bar{\sigma}=\sqrt{1/n_s\sum_{i=1}^{n_s} (R_i-\bar{R})^2 }$ is the standard deviation of the sample , we derive 
\begin{equation}
\frac{\bm{\lambda}^{\star}T +d_t^s(\pi_s)- C_2}{\bm{\lambda}^{\star}T -C_4}\leq \hat{\rho}(\pi_s,\pi_t,T)\leq \frac{\bm{\lambda}^{\star}T +d_t^s(\pi_s)- C_1}{\bm{\lambda}^{\star}T -C_3}
\end{equation}
with $C_1, C_2, C_3,$ and $C_4$ are constants. Consequently, for negative transfer:
$$\hat{\rho}(\pi_s,\pi_t,T) \geq \frac{\bm{\lambda}^{\star}T +d_t^s(\pi_s)- C_2}{\bm{\lambda}^{\star}T -C_4}> 1.$$ Then, assuming enough samples, negative transfer occurs if: 
\begin{align*}
&d_{t}^{s}(\pi_{s}) \\
&> \left\{\mathbb{E}_{\pi_s}\left[\sum_{t=0}^{T} \mathcal{R}_s (\bm{s}_{t},\bm{a}_{t})\right]-\mathbb{E}_{\pi_t}\left[\sum_{t=0}^{T} \mathcal{R}_t (\bm{s}_{t},\bm{a}_{t})\right]\right\} \numberthis \label{eq:condition}
\end{align*}

The condition sheds light on the negative transfer in the sense of metric notation and provides a formal way to determine negative transfer. First, if the condition in Eq. \ref{eq:condition} holds after evaluation, researchers should avoid the source policy $\pi_s$ to the target tasks since it may cause negative transfer. Second, if 
the researchers have enough expert knowledge about their working domain and transfer information, usually they can avoid this evaluation phase in practice. 
In short, Eq. \ref{eq:condition} provides a formal way to understand negative transfer and justify the intuition (adopt high quality source knowledge and avoid bad teachers ) in the transfer practice.

\section{Experimental Results}
\begin{figure*}[t]
	\centering
	\begin{subfigure}{.6\columnwidth}
		\centering
		\includegraphics[width=\columnwidth]{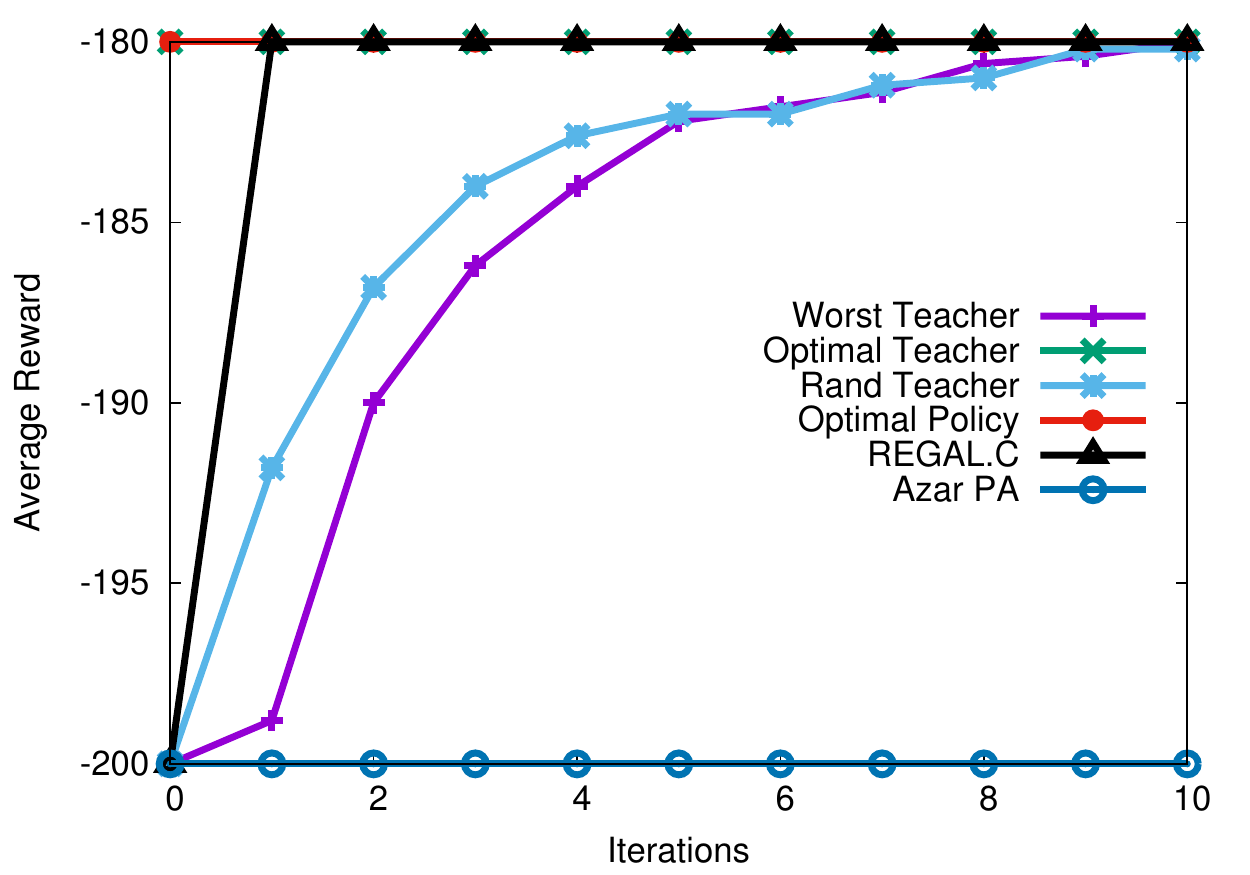}
		\caption{Grid World}
		\label{fig:gridworld}
	\end{subfigure}
	\begin{subfigure}{.6\columnwidth}
		\centering
		\includegraphics[width=\columnwidth]{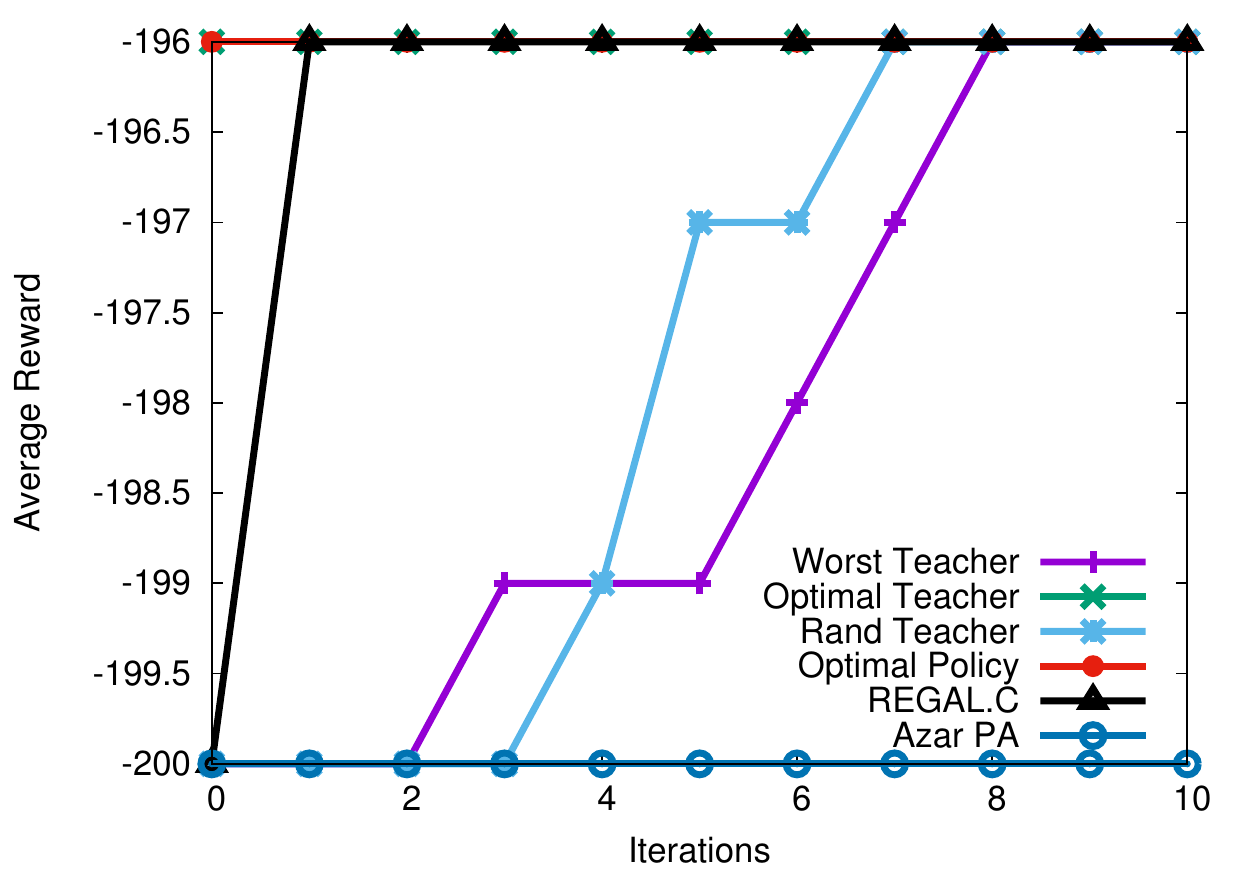}
		\caption{Combination Lock}
		\label{fig:comblock}
	\end{subfigure}
	\begin{subfigure}{.6\columnwidth}
		\centering
		\includegraphics[width=\columnwidth]{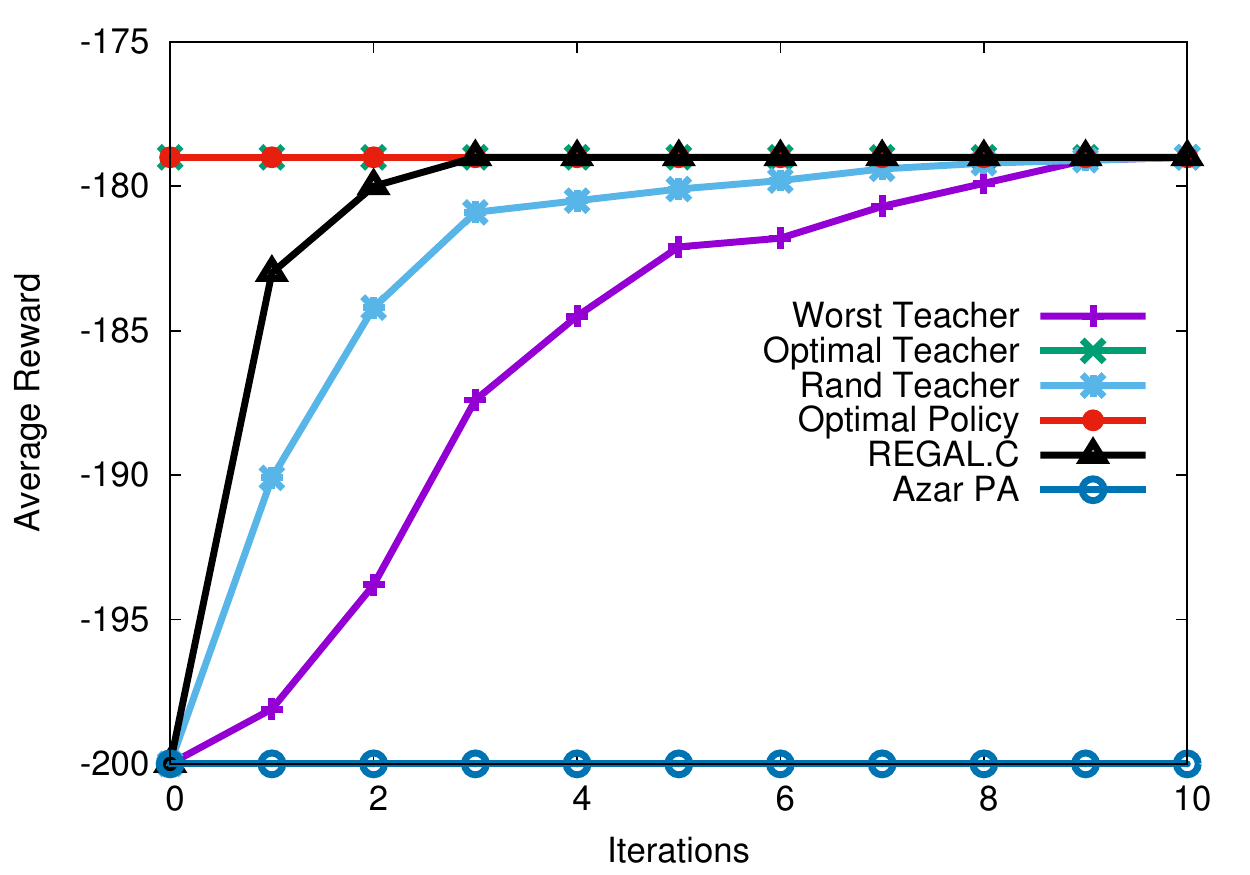}
		\caption{Block Dude}
		\label{fig:blockdude}
	\end{subfigure}
	\caption{Our method with optimal teacher has similar performance as the optimal policy. And the REGAL algorithm (no advice) outperform random teacher and worst teacher group which justifies that the poorer teachers do harm the learning. Azar's method depends on the quality of the teachers --- when the teachers are very poor, the algorithm shows no learning.}
	\label{fig:avaragereawrd}
	\vspace{-5pt}
\end{figure*}

Given the above theoretical successes, this section provides empirical validation on three domains:

\textbf{Combination Lock:} We use the domain described in Figure \ref{fig:mdpexample} which is a variation from \cite{whitehead1991complexity}. The experimental setting follows the caption description.

\textbf{Grid World} is an RL benchmark in which an agent has to navigate an $m \times m$ grid world with the goal of reaching a goal state. We employ an $11 \times 11$ grid world with a four room layout as introduced in~\citeauthor{sutton1998introduction}~\shortcite{sutton1998introduction}. The agent begins in the lower left corner of the map and navigates to the goal state being the upper right corner. To navigate, the agent has access (in each cell) to four actions transitioning it to the: north, south, west and east. Applying an action, it then transitions in that direction with a probability of 0.8 and in the other three with a probability of 0.2. In case the direction is blocked, the agent stays in the same state. Finally, the agent receives a reward of $0$ once reaching the goal state and a reward of $-1$ in all others.

\textbf{Block Dude} is a game where an agent again navigates a maze to reach a goal state. Reaching the goal directly is impossible due to the presence of blocks restricting its movement. The agent, however, can move to the left, right, and upwards. To reach the goal state, it needs to pick-up blocks and relocate them in correct positions. We use the default level 1 BURLAP \cite{burlap} in which there are two blocks and $3 \times 25$ maze. The agent receives a reward of $+1$ in the goal state and a reward of $-1$ in all other states.

\begin{figure}
	\centering
	\includegraphics[width=.7\linewidth]{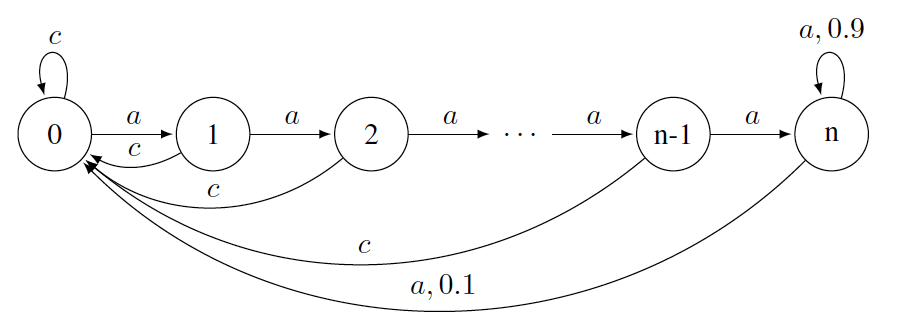}
	\caption{
		There are $n+1$ states in the MDP. The last state has only one action and the rest have two. The agent receives reward $-1$ for all actions, except when taking action $a$ in state $n$, $\mathcal{R}(n,a)=1$. The agent stays in state $n$ with probability $0.9$ and goes to state $0$ with probability $0.1$. The optimal policy is to take action $a$ in each state. Since there are $n+1$ states, the budget $\mathfrak{B}$ is at least $n+1$ to achieve zero regret. 
	}
	\label{fig:mdpexample}
\end{figure}

\subsection{Experimental Setup \& Results}
To construct the grand teacher, we set the total number of teachers $k=10$. For each teacher, the budget, $\bm{b}_{i}$, is set to the total number of states. In Algorithm \ref{alg:learning_dagger}, the maximum number of iterations and the size of each dataset, $\mathcal{D}_i$, were set to $10$ and  $200$, respectively. Values of $p^{i}=0.5^{i}$ for $i=1,\dots, 10$ were used to determine $\bm{\beta}_i=p^i$. For Algorithm~\ref{alg:REGAL.C}, the confidence $\bm{\delta}$ was set to $0.8$ and $H$ to $1000$. The optimal gain $\lambda^{\star}$ and the optimal bias vector $\bm{h}$ can be approximated using the value function $\mathcal{V}$ \cite{puterman2005markov}. Let $\mathcal{V}^l$ be the value function at iteration $l$, $l=0,1,\dots$. The optimal gain $\bm{\lambda}^{\star}\approx \text{sp}(\mathcal{V}^{l+1}-\mathcal{V}^{l})$, where $\text{sp}(\mathcal{V}^{l})=\max_{\bm{s}\in \mathcal{S}}\mathcal{V}^{l}(\bm{s})-\min_{\bm{s}\in \mathcal{S}}\mathcal{V}^l(\bm{s})$ and the optimal bias vector $\bm{h}^{\star}\approx \mathcal{V}^{l}-l\bm{\lambda}^{l}$, when $l$ is large enough. To smooth the natural variance in the student's performance, each learning curve is averaged over $10$ independent trials of student learning. To better evaluate our method, we adopt six experimental settings by considering different teachers and learning algorithms. For teachers, we consider three forms. The first, referred to as ``optimal teacher'' provides optimal actions and is used by the grand teachers. The second, referred to as ``worst teacher'' advices the student to take actions with the lowest Q-values, while the last randomly selects action suggestions from the set of allowed moves. We also compare our method to REGAL.C (no advice), optimal policy (without learning), and Azar's method~\cite{azar2013regret}. Please note that Azar's method can not converge to the optimal policy and suffers loss as its performance is restricted by the teacher. 

Performance, measured by the average reward, is reported in Figure \ref{fig:avaragereawrd}. First, it is clear that given optimal teachers, our method exactly traces the optimal policy achieving a regret of 0. It is also important to note that in all three domains, even if the teacher was not optimal, and contrary to current techniques, our method is capable of acquiring optimal behavior. This is achievable as our method allows for learning within the multiple teacher framework. 



\section{Related Work on Transfer Learning}
Few theoretical results on transfer and policy advice have been achieved. Closest to this work is that in~\citeauthor{torrey2013teaching}~\shortcite{torrey2013teaching}, where the authors only provide empirical validations to their approach without drawing on any theoretical analysis. Given the theoretical derivations in this paper, we in fact note that the method~\cite{torrey2013teaching} is a special case of ours considering only one-teacher advice models.

Another method considering advice under multiple teachers is that in~\citeauthor{azar2013regret}~\shortcite{azar2013regret}. \citeauthor{azar2013regret} propose a method capable of selecting the best policy from a set of teacher policies and derive sub-linear regret of the form $\mathcal{O}(\sqrt{T})$ with $T$ being the total number of rounds. One drawback of their method, however, is the assumption of a ``good-enough'' teacher which can guide the student to optimality. Such a method may suffer huge regret if the overall quality of teacher policies is poor. It also can not obtain better policies than those of the teacher. Our algorithm remedies these problems by allowing agents to further improve, which gives them the opportunity to surpass the teacher's performance. 

Human advice is also a good source of policy advice. Usually, this method adopts the human advice as the teacher's policy to improve the learning performance. However, these works focus on empirical validations \cite{cakmak2012algorithmic,griffith2013policy}. 

Probabilistic policy reuse is similar to our method in which the algorithm follows its own knowledge with probability $1-\epsilon$ and teacher's policy with probability $\epsilon$ \cite{fernandez2006probabilistic}. However, $\epsilon$ is not decaying over time, making the algorithm divergent if teacher policies are not optimal. 
\citeauthor{isbell:policy:2015} introduce a policy shaping algorithm using human teachers, but focus on providing rewards rather than action advice \cite{isbell:policy:2015}. Both of these works rely solely on empirical results.

Work on transfer for RL is also related to this paper, where we can consider policy advice as an instance of transferring from teachers to students~\cite{lazaric2012transfer}. Here, \citeauthor{ferrante2008transfer}, for instance, propose a method to transfer high quality samples from source to target tasks using bi-simulation measures~\cite{ferrante2008transfer}. Their method only transfers samples once, while our approach gradually provides advice to the student. Due to space constraints, we refer the reader to~\citeauthor{taylor2009transfer}~\shortcite{taylor2009transfer} for a comprehensive survey.

Lifelong reinforcement learning has drawn significant attention to the transfer community recently. Brunskill and Li studied online discovery problems in a lifelong learning setting~\cite{DBLP:journals/corr/BrunskillL15}. Bou-Ammar \textit{et al.} also studied such a problem and introduced constraints on the policy to compute ``safe'' policies \cite{DBLP:journals/corr/Bou-AmmarTE15}. Contrary to these works, in this paper, we focus on the single agent setting operating within one task. 

Finally, Learning from Demonstration  \cite{argall2009survey} (LfD) is also related to our work, but LfD usually assumes that the expert is optimal and the student only tries to mimic the expert. 

\section{Conclusion and Future Work}
\label{sec:conclusion}
In this paper, we formally defined the multi-teacher advice model and introduced a new algorithm which leverages teacher and student's own knowledge in the weakly communicating MDPs. We theoretically analyzed our algorithm and showed, for the first time, that the agent can achieve optimality even when starting from non-optimal teachers. Our results provide a theoretical justification for the intuition that ``bad'' teachers can hurt the learning process of the student. Also, we formally established the condition of negative transfer, shedding light on future transfer learning research, where for example, researchers can choose ``good teachers'' based on the Eq \ref{eq:condition} and avoid negative transfer with prior expert knowledge.

In future, we plan on adopting other online reinforcement learning algorithms (e.g., REGAL.D \cite{bartlett2009regal}, R-max \cite{brafman2003r}, or $E^3$ \cite{kearns2002near}) to replace REGAL.C. We will provide better methods to construct the ``grand-teacher'' without exploring the whole MDP. Also, extensions to large-scale MDPs may be an interesting direction for future research as well.

\section{Acknowledgements}
This research has taken place in part at the Intelligent Robot
Learning (IRL) Lab, Washington State University. IRL research
is supported in part by grants AFRL FA8750-14-1-
0069, AFRL FA8750-14-1-0070, NSF IIS-1149917, NSF IIS-
1319412, USDA 2014-67021-22174, and a Google Research
Award. 

\bibliographystyle{named}
\bibliography{General}
\appendix

\section{Proof of Theorem 1}
\begin{proof}[\textbf{Proof of Theorem 1}]
	Let $\mathcal{Z}_i^r$ denotes the event that the $i$-th state, $s_i$, is not visited in the first $r$ explorations.  $$P[\mathcal{Z}_i^r]=(1-\frac{1}{|\mathcal{S}|})^r\leq e^{-r/|\mathcal{S}|}.$$ If we choose $r=c|\mathcal{S}|\log |\mathcal{S}|$, where $c$ is a constant,
	$$P[\mathcal{Z}_i^r]\leq e^{-r/|\mathcal{S}|}=e^{-c\log |\mathcal{S}|}=|\mathcal{S}|^{-c}.$$
	Let $\mathcal{T}$ be the number of steps at least one of all state is not visited.
	\begin{align*}
	P[\mathcal{T} > c|\mathcal{S}|\log |\mathcal{S}|]&=P\left[\exists i: \mathcal{Z}_i^r \text{ s. t. } r=c|\mathcal{S}|\log |\mathcal{S}|\right] \\
	&=P\left[\bigcup_{i=1}^{|\mathcal{S}|}  \mathcal{Z}_i^{c|\mathcal{S}|\log |\mathcal{S}|}\right] \\
	\shortintertext{Apply the union bound}
	&\leq \bigcup_{i=1}^{|\mathcal{S}|} P\left[ \mathcal{Z}_i^{c|\mathcal{S}|\log |\mathcal{S}|}\right] \\
	&\leq|\mathcal{S}|\cdot |\mathcal{S}|^{-c}=|\mathcal{S}|^{1-c}
	\end{align*}
	Set $|\mathcal{S}|^{1-c}=\delta$, we have $$\mathcal{T}\leq |\mathcal{S}| \log \frac{|\mathcal{S}|}{\delta}=O\left(|\mathcal{S}| \log \frac{|\mathcal{S}|}{\delta}\right)$$
	with probability at least $1-\delta$.
\end{proof}
\section{Proofs of Theorem 2 and Corollary 1}
Next, we will prove Theorem \ref{thm:regalregret}. To review some notation, we use $N_i(\bm{s},\bm{a})$ to denote the number of times a state-action pair $(\bm{s},\bm{a})$ at iteration $i$. And $v_i(\bm{s},\bm{a})=N_{i+1}(\bm{s},\bm{a})-N_{i}(\bm{s},\bm{a})$ denotes the number of times a state-action pair $(\bm{s},\bm{a})$ is visited during iteration $i$.
Let $\Delta_i$ be the regret incurred in iteration $i$,
\begin{equation}
\label{eq:expectedregret}
\Delta_i=\sum_{\bm{s},\bm{a}}v_i(\bm{s},\bm{a})\left(\bm{\lambda}^*-\mathcal{R}(\bm{s},\bm{a})\right).
\end{equation}
The total regret equals
$$\sum_{i=1}^m\Delta_i,$$ where $m$ is the number of iterations in Algorithm 3. Auer \textit{et al.} show that $m\leq |\mathcal{S}||\mathcal{A}|\log (8T/|\mathcal{S}||\mathcal{A}|)$ if $T\geq |\mathcal{S}||\mathcal{A}|$ \cite{auer2009near}. Let
\begin{equation*}
X_i=\begin{cases}
1& \text{with probabiltiy } 1-\beta_i \\
0& \text{with probabiltiy } \beta_i. 
\end{cases}
\end{equation*}
That is, $X_i$ is the indicator random variable for the iteration $i$.
\begin{lemma}
	\label{lm:decompsedregret}
	Consider an iteration $i$. Then, we can decompose the regret as two components, $$\Delta_i=X_i\tilde{\Delta}_i+(1-X_i)\Delta^{\mathfrak{T}}_i,$$ where $\tilde{\Delta}_i$ and $\Delta^{\mathfrak{T}}_i$ are the regret incurred in iteration $i$ following the policy $\pi_i$ and the grand-teacher policy $\pi^{\mathfrak{T}}$, respectively.	
\end{lemma}
\begin{proof}
	According to Equation (\ref{eq:expectedregret}), we have
	\begin{align*}
	\Delta_i&=\sum_{\bm{s},\bm{a}}v_i(\bm{s},\bm{a})\left(\bm{\lambda}^*-\mathcal{R}(\bm{s},\bm{a})\right)\\
	\shortintertext{At each decision step, the student agent either follows $\pi_i$ or the teacher $\pi^{\mathfrak{T}}$}
	&=\left(\sum_{\bm{s}}X_iv_i(\bm{s},\pi_i(\bm{s}))\left(\bm{\lambda}^*-\mathcal{R}(\bm{s},\pi_i(\bm{s}))\right)\right) \\
	&\mspace{20mu}+\left(\sum_{\bm{s}}(1-X_i)v_i(\bm{s},\pi^{\mathfrak{T}}(\bm{s}))\left(\bm{\lambda}^*-\mathcal{R}(\bm{s},\pi^{\mathfrak{T}}(\bm{s}))\right)\right) \\
	\shortintertext{$X_i$ is not related to state $\bm{s}$}
	&=X_i\left(\sum_{\bm{s}}v_i(\bm{s},\pi_i(\bm{s}))\left(\bm{\lambda}^*-\mathcal{R}(\bm{s},\pi_i(\bm{s}))\right)\right) \\
	&\mspace{20mu}+(1-X_i)\left(\sum_{\bm{s}}v_i(\bm{s},\pi^{\mathfrak{T}}(\bm{s}))\left(\bm{\lambda}^*-\mathcal{R}(\bm{s},\pi^{\mathfrak{T}}(\bm{s}))\right)\right) \\
	&=X_i\tilde{\Delta}_i+(1-X_i)\Delta^{\mathfrak{T}}_i
	\end{align*}
	where $\tilde{\Delta}_i=\sum_{\bm{s}}v_i(\bm{s},\pi_i(\bm{s}))\left(\lambda^*-\mathcal{R}(\bm{s},\pi_i(\bm{s}))\right)$ is the regret that the student agent only follows $\pi_i$ and $\Delta^{\mathfrak{T}}_i=\sum_{\bm{s}}v_i(\bm{s},\pi^{\mathfrak{T}}(\bm{s}))\left(\bm{\lambda}^*-\mathcal{R}(\bm{s},\pi^{\mathfrak{T}}(\bm{s}))\right)$ is the regret that the student agent only follows the grand teacher's advice $\pi^{\mathfrak{T}}$.
\end{proof}

\begin{lemma}
	\label{lm:goodmdpregret}
	The total regret $\sum_{i=1}^m\Delta_i$ has following upper bound,
	\begin{equation}
	\label{eq:twoparts}
	\sum_{i=1}^m\Delta_i\leq \max_{1\leq i\leq m}\{\rho_i\}\left(\sum_{i=1}^m\tilde{\Delta}_i-\sum_{i=1}^mX_i\tilde{\Delta}_i\right)+\sum_{i=1}^mX_i\tilde{\Delta}_i,
	\end{equation}
	where $\rho_i\geq 0$ is the ratio such that $\Delta^{\mathfrak{T}}_i=\rho_i\tilde{\Delta}_i$, $i=1,\dots,m$.
\end{lemma}
\begin{proof}
	Consider an iteration $i$. Lemma \ref{lm:decompsedregret} implies $$\Delta_i=X_i\tilde{\Delta}_i+(1-X_i)\Delta^{\mathfrak{T}}_i.$$ Therefore,
	\begin{align*}
	\sum_{i=1}^m\Delta_i&=\sum_{i=1}^m\left(X_i\tilde{\Delta}_i+(1-X_i)\Delta^{\mathfrak{T}}_i\right) \\
	&=\sum_{i=1}^m\left(X_i\tilde{\Delta}_i+(1-X_i)\rho_i\tilde{\Delta}_i\right)  \\
	\shortintertext{Assume that $\Delta^{\mathfrak{T}}_i=\rho_i\tilde{\Delta}_i$. Note that we introduce the regret ratio $\rho_i$ in Definition 2.}
	&=\sum_{i=1}^m\left(\rho_i\tilde{\Delta}_i+X_i\tilde{\Delta}_i-X_i\rho_i\tilde{\Delta}_i\right)\\
	&=\sum_{i=1}^m\rho_i\tilde{\Delta}_i-\sum_{i=1}^mX_i\rho_i\tilde{\Delta}_i+\sum_{i=1}^mX_i\tilde{\Delta}_i \\
	&\leq\max_{1\leq i\leq m}\{\rho_i\}\left(\sum_{i=1}^m\tilde{\Delta}_i-\sum_{i=1}^mX_i\tilde{\Delta}_i\right)+\sum_{i=1}^mX_i\tilde{\Delta}_i
	\end{align*}
	Hence, we will bound $\sum_{i=1}^m\tilde{\Delta}_i$ and $\sum_{i=1}^mX_i\tilde{\Delta}_i$, separately.
\end{proof}

\begin{lemma}[Theorem $2$ in \cite{bartlett2009regal} ]
	\label{lm:regalregret1}
	With probability at least $1-\delta$, the total regret $\sum_{i=1}^m\tilde{\Delta}_i$ of REGAL.C algorithm satisfies $$\sum_{i=1}^m\tilde{\Delta}_i= \mathcal{O}\left(H|\mathcal{S}|\sqrt{|\mathcal{A}|T\log (|\mathcal{A}|T/\delta)}\right),$$ where $0<\beta_i<1$ is the decaying variable in Algorithm 3. 
\end{lemma}
The above Lemmas gives the upper bound of $\sum_{i=1}^m\tilde{\Delta}_i$, which is given by Theorem $2$ in ~\cite{bartlett2009regal}. For $\sum_{i=1}^mX_i\tilde{\Delta}_i$, we have following bound:
\begin{lemma}
	\label{lm:modifiedregalregret}
	With probability at least $1-\delta$, $$\sum_{i=1}^mX_i\tilde{\Delta}_i= \mathcal{O}\left(\max_{1\leq i\leq m}\{1-\beta_i\}H|\mathcal{S}|\sqrt{|\mathcal{A}|T\log (|\mathcal{A}|T/\delta)}\right)$$
\end{lemma}
\begin{proof}
	Barlett and Tewari gives following bound \cite{bartlett2009regal}, with probability $1-\delta/2$
	\begin{align*}
	\sum_{i=1}^m\tilde{\Delta}_i&\leq H\left(\sum_{i=1}^{m}\sum_{\bm{s},\bm{a}}\frac{2v_{i}(\bm{s},\bm{a})}{\sqrt{N_i(\bm{s},\bm{a})}}\sqrt{12|\mathcal{S}|\log(4|\mathcal{A}|T/\delta)}\right.\\
	&\left.+\sqrt{2T\log(2/\delta)}+m+\sqrt{T}\right)
	\end{align*}
	With this result and $X_i\leq 1$, 
	\begin{align*}
	\sum_{i=1}^mX_i\tilde{\Delta}_i&\leq H\left(\sum_{i=1}^{m}X_i\sum_{\bm{s},\bm{a}}\frac{2v_{i}(\bm{s},\bm{a})}{\sqrt{N_i(\bm{s},\bm{a})}}\sqrt{12|\mathcal{S}|\log(2|\mathcal{A}|T/\delta)}\right.\\
	&\left.+\sqrt{2T\log(1/\delta)}+m+\sqrt{T}\right) \numberthis\label{eq:xidelta}
	\end{align*}
	Since the first term $$\sum_{i=1}^{m}X_i\sum_{\bm{s},\bm{a}}\frac{2v_{i}(\bm{s},\bm{a})}{\sqrt{N_i(\bm{s},\bm{a})}}\sqrt{12|\mathcal{S}|\log(2|\mathcal{A}|T/\delta)}$$ dominates the right-hand side, we need to bound it carefully. Let $$c_i=\sum_{\bm{s},\bm{a}}\frac{2v_{i}(\bm{s},\bm{a})}{\sqrt{N_i(\bm{s},\bm{a})}}\sqrt{12|\mathcal{S}|\log(2|\mathcal{A}|T/\delta)}$$ and $$Z_i=X_ic_i,$$ $Z_1,Z_2,\dots Z_m$ are independent random variables with $Z_i$ such that $0\leq Z_i\leq c_i$. Then apply Hoeffding's inequality \cite{mohri2012foundations}, we obtain, with probability at least $1-\delta/2$, 
	\begin{equation}
	\label{eq:zi}
	\sum_{i=1}^m Z_i\leq \mathbb{E}\left[\sum_{i=1}^{m} Z_i\right]-\sqrt{\frac{\sum_{i=1}^{m}c_i\log 2/\delta}{2}}.
	\end{equation}
	Due to the linearity of expectation, $$\mathbb{E}\left[\sum_{i=1}^{m} Z_i\right]=\sum_{i=1}^{m} \mathbb{E}\left[Z_i\right]=\sum_{i=1}^{m} \mathbb{E}\left[X_i\right]c_i=\sum_{i=1}^{m} (1-\beta_i)c_i,$$
	Combining this with Eq. (\ref{eq:zi}),
	\begin{align*}
	\sum_{i=1}^{m}X_ic_i&=\sum_{i=1}^m Z_i\\
	&\leq \sum_{i=1}^{m}(1-\beta_i)c_i-\sqrt{\frac{\sum_{i=1}^{m}c_i\log 2/\delta}{2}}\\
	&\leq \sum_{i=1}^{m}(1-\beta_i)c_i \\
	&= \sum_{i=1}^{m}(1-\beta_i)\sum_{\bm{s},\bm{a}}\frac{2v_{i}(\bm{s},\bm{a})}{\sqrt{N_i(\bm{s},\bm{a})}}\sqrt{12|\mathcal{S}|\log(4|\mathcal{A}|T/\delta)}\\
	&\leq \max_{1\leq i\leq m}\{1-\beta_i\}\sum_{i=1}^{m}\sum_{\bm{s},\bm{a}}\frac{2v_{i}(\bm{s},\bm{a})}{\sqrt{N_i(\bm{s},\bm{a})}}\sqrt{12|\mathcal{S}|\log(4|\mathcal{A}|T/\delta)}.
	\end{align*}
	Plugging it into Eq. (\ref{eq:xidelta}), we get
	\begin{align*}
	&\sum_{i=1}^mX_i\tilde{\Delta}_i \\
	&\leq H\left(\max_{1\leq i\leq m}\{1-\beta_i\}\sum_{i=1}^{m}\sum_{\bm{s},\bm{a}}\frac{2v_{i}(\bm{s},\bm{a})}{\sqrt{N_i(\bm{s},\bm{a})}}\sqrt{12|\mathcal{S}|\log(4|\mathcal{A}|T/\delta)}\right.\\
	&\left.+\sqrt{2T\log(2/\delta)}+m+\sqrt{T}\right) \numberthis\label{eq:removex_i}
	\end{align*}
	Eq ($20$) in \cite{auer2009near} gives 	$$\sum_{i=1}^{m}\sum_{\bm{s},\bm{a}}\frac{v_{i}(\bm{s},\bm{a})}{\sqrt{N_i(\bm{s},\bm{a})}} \leq (\sqrt{2}+1)\sqrt{|\mathcal{S}||\mathcal{A}|T},$$
	Substitute this for Eq. (\ref{eq:removex_i}) and Eq. (20) in ~\cite{auer2009near}  ($m\leq |\mathcal{S}||\mathcal{A}|\log_2 (8T/|\mathcal{S}||\mathcal{A}|)$ if $T\geq |\mathcal{S}||\mathcal{A}|$), yielding,
	\begin{align*}
	\sum_{i=1}^mX_i\tilde{\Delta}_i= \mathcal{O}\left(\max_{1\leq i\leq m}\{1-\beta_i\}H|\mathcal{S}|\sqrt{|\mathcal{A}|T\log (|\mathcal{A}|T/\delta)}\right),
	\end{align*}
	with probability $1-\delta$.
\end{proof}

\begin{proof}[\textbf{Proof of Theorem 2}] 
	Using Lemma \ref{lm:regalregret1} and \ref{lm:modifiedregalregret}, with probability at least $1-\delta$, 
	$$\sum_{i=1}^m\tilde{\Delta}_i=\mathcal{O}\left(H|\mathcal{S}|\sqrt{|\mathcal{A}|T\log (|\mathcal{A}|T/\delta)}\right),$$ and
	$$\sum_{i=1}^mX_i\tilde{\Delta}_i= \mathcal{O}\left(\max_{1\leq i\leq m}\{1-\beta_i\}H|\mathcal{S}|\sqrt{|\mathcal{A}|T\log (|\mathcal{A}|T/\delta)}\right).$$
	Combining with Eq. \ref{eq:twoparts}, we have,
	\begin{align*}
	\sum_{i=1}^m\Delta_i&\leq \max_{1\leq i\leq m}\{\rho_i\}\mathcal{O}\left(H|\mathcal{S}|\sqrt{|\mathcal{A}|T\log (|\mathcal{A}|T/\delta)}\right) \\
	&-\max_{1\leq i\leq m}\{\rho_i\} \mathcal{O}\left(\max_{1\leq i\leq  m}\{1-\beta_i\}H|\mathcal{S}|\sqrt{|\mathcal{A}|T\log (|\mathcal{A}|T/\delta)}\right)\\
	&+ \mathcal{O}\left(\max_{1\leq i\leq m}\{1-\beta_i\}H|\mathcal{S}|\sqrt{|\mathcal{A}|T\log (|\mathcal{A}|T/\delta)}\right)\\
	\shortintertext{Let $\bm{\rho}=\max_{1\leq i\leq m}\{\rho_i\}$ and $1-\bm{\beta}=\max_{1\leq i\leq m}\{1-\beta_i\}$}
	&=\mathcal{O}\left((1-\bm{\beta}+\bm{\rho\beta})H|\mathcal{S}|\sqrt{|\mathcal{A}|T\log (|\mathcal{A}|T/\delta)}\right)
	\end{align*}
\end{proof}
\begin{proof}[\textbf{Proof of Corollary 1}]
	If the teacher is optimal, then $\Delta^T_i=0=\rho_i\tilde{\Delta}_i$, that is $\rho_i=0$ for $i=1,\dots,m$. Therefore, $\bm{\rho}=0$. The result follows.
\end{proof}
\section{Domains GUI Examples}

Here we provide the GUI examples of Grid World and Block Dude which are used as the experimental domains in the main paper.

\begin{figure}[h]
	\centering
	\includegraphics[width=0.8\linewidth]{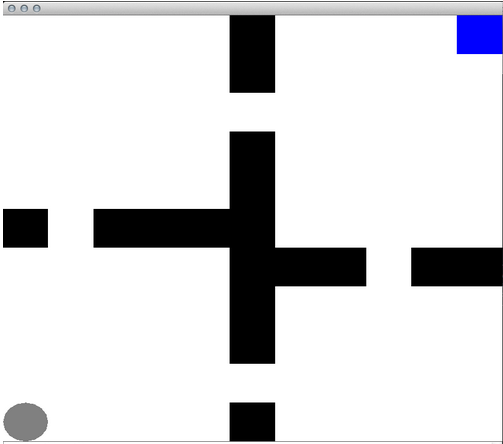}
	\caption{An GUI example of the Grid World. The map is separated into four rooms by the wall. An agent in the lower left corner tries to reach the goal state in the upper right corner.}
	\label{fig:gridworldexmaple}
\end{figure}

\begin{figure}[h]
	\centering
	\includegraphics[width=0.8\linewidth]{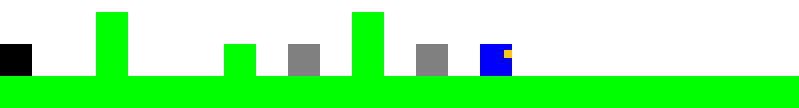}
	\caption{An GUI example of the Block Dude. The agent needs to move the blocks to assist it to reach the goal state.}
	\label{fig:blockdudeexample}
\end{figure}

\end{document}